\documentclass[12pt]{article}

\usepackage{times,amsmath,amsfonts,stmaryrd,amssymb,enumitem, amsthm}
\usepackage{algorithm,booktabs}

\usepackage{algorithmic}

\usepackage{pifont}

\usepackage{tikz}
\usetikzlibrary{arrows}
\usetikzlibrary{calc}
\usetikzlibrary{shapes}
\usetikzlibrary{fit}
\usetikzlibrary{matrix} 
\usetikzlibrary{positioning}
\usetikzlibrary{decorations.pathreplacing}

\newtheorem{theorem}{Theorem}[section]
\newtheorem{lemma}[theorem]{Lemma}

\newtheorem{definition}[theorem]{Definition}

\newtheorem{example}[theorem]{Example}

\newcommand{\exref}[1]{Example \ref{ex:#1}}
\newcommand{\appref}[1]{Appendix \ref{ap:#1}}
\renewcommand{\eqref}[1]{Eq.~(\ref{eq:#1})}
\newcommand{\algref}[1]{Alg.~\ref{alg:#1}}
\newcommand{\secref}[1]{Section \ref{sec:#1}}
\newcommand{\thmref}[1]{Theorem \ref{thm:#1}}
\newcommand{\lemref}[1]{Lemma \ref{lem:#1}}

\renewcommand{\P}{\mathbb{P}}
\newcommand{\E}{\mathbb{E}}

\newcommand{\reals}{\mathbb{R}}

\newcommand{\half}{{\frac12}}
\newcommand{\ceil}[1]{{\lceil #1\rceil}}
\newcommand{\floor}[1]{\left\lfloor #1\right\rfloor}

\newcommand{\one}{\mathbb{I}}

\DeclareMathOperator*{\argmax}{argmax}

\newcommand{\cA}{\mathcal{A}}

\newcommand{\cD}{\mathcal{D}}

\newcommand{\cG}{\mathcal{G}}
\newcommand{\cH}{\mathcal{H}}

\newcommand{\cU}{\mathcal{U}}

\newcommand{\cX}{\mathcal{X}}
\newcommand{\cY}{\mathcal{Y}}

\newcommand{\mysharedbib}{\string~/Dropbox/Research/useful/shared}

\usepackage[margin=3cm]{geometry}
\usepackage{times}
\usepackage[round]{natbib}

\newcommand{\secp}{\mathrm{SecPr}}

\newcommand{\distset}{\textsc{ds}}
\newcommand{\outpool}{\mathsf{sel}_p}
\newcommand{\outpoolr}{\mathsf{pairs}_p}
\newcommand{\outstream}{\mathsf{sel}_s}
\newcommand{\outstreamr}{\mathsf{pairs}_s}

\title{Interactive Algorithms: from Pool to Stream}

\author{Sivan Sabato and Tom Hess\\
Department of Computer Science\\
Ben-Gurion University of the Negev}
\date{}

\newcommand{\nselect}{N_{\mathsf{sel}}}
\newcommand{\niter}{N_{\mathsf{iter}}}

\begin{document}
\maketitle

\begin{abstract}
 We consider interactive algorithms in the pool-based setting, and in the stream-based setting. Interactive algorithms observe suggested elements (representing actions or queries), and interactively select some of them and receive responses. Pool-based algorithms can select elements at any order, while stream-based algorithms observe elements in sequence, and can only select elements immediately after observing them. We assume that the suggested elements are generated independently from some source distribution, and ask what is the stream size required for emulating a pool algorithm with a given pool size. We provide algorithms and matching lower bounds for general pool algorithms, and for utility-based pool algorithms. We further show that a maximal gap between the two settings exists also in the special case of active learning for binary classification.
\end{abstract}

\section{Introduction}

\emph{Interactive algorithms} are algorithms which are presented with input in the form of suggested elements (representing actions or queries), and iteratively select elements, getting a response for each selected element. The reward of the algorithm, which is application-specific, is a function of the final set of selected elements along with their responses. Interactive algorithms are used in many application domains, including, for instance, active learning \citep{MccallumNi98}, interactive sensor placement \citep{GolovinKr11}, summarization \citep{SinglaTsKr16} and promotion in social networks \citep{GuilloryBi10}. As a specific motivating example, consider an application in which elements represent web users, and the algorithm should select up to $q$ users to present with a free promotional item. For each selected user, the response is the observed behavior of the user after having received the promotion, such as the next link that the user clicked on. The final reward of the algorithm depends on the total amount of promotional impact it obtained, as measured by some function of the set of selected users and their observed responses. Note that the algorithm can use responses from previous selected users when deciding on the next user to select.

We consider two interaction settings for interactive algorithms: The \emph{pool-based} setting and the \emph{stream-based} setting. In the pool-based setting, the entire set of suggested elements is provided in advance to the algorithm, which can then select any of the elements at any order. For instance, in the web promotion example, there might be a set of users who use the website for an extended period of time, and any of them can be approached with a promotion. In the stream-based setting, elements are presented to the algorithm in sequence, and the algorithm must decide immediately after observing an element, whether to select it or not. In the web promotion example, this is consistent with a setting where users access the website for single-page sessions, and so any promotion must be decided on immediately when the user is observed.

The stream-based setting is in general weaker than the pool-based setting. Nonetheless, it is important and
useful: In many real-life scenarios, it is not possible to postpone selection of elements, for instance due to storage and
retrieval constraints, or because of timing constraints. This is especially pertinent when the data stream is
real-time in nature, such as in streaming document classification
\citep{BougueliaBeBe13}, in spam filtering \citep{ChuZiLiAcBe11}, in web
streams such as Twitter \citep{SmailovicGrLaZn14}, in video
surveillance \citep{LoyHoXiSh12} and with active sensors
\citep{Krishnamurthy02}. 

In this work, our goal is to study the relationship between these two important settings. Both of these settings have been widely studied in many contexts. 
In active learning, both settings have been studied in classic works \citep{CohnAtLa94,LewisGa94}. Works that address mainly the stream-based setting include, for instance, \citet{BalcanBeLa09,Hanneke11,Dasgupta12, BalcanLo13,SabatoMu14}. Some theoretical results hold equally for the stream-based and the pool-based settings \citep[e.g.,][]{BalcanLo13,HannekeYa15}. 

Several near-optimal algorithms have been developed for the pool-based setting \citep{Dasgupta05, GolovinKr11, GolovinKrRa10, Hanneke07b, SabatoSaSr13, GonenSaSh13b, CuongLeYe14}. The pool-based setting is
also heavily studied in various active learning applications \citep[e.g.,][]{TongKoller02,TongCh01,MitraMuPa04,GosselinCo08,CebronBe09,GuoSiStKo13}.
General interactive algorithms have also been studied in both a pool-based setting \cite[e.g.,][]{GolovinKr11,GuilloryBi10,DeshpandeHeKl14} and in stream-based settings \cite[e.g.,][]{DemaineInMaVa14,ArlottoMoSt14,StreeterGo09,GolovinFaKr10}. Note that unlike some works on interactive algorithms, in our stream-based setting, the only direct restriction is on the timing of selecting elements. We do not place restrictions on storage space or any other resources. 

To study the relationship between the pool-based setting and the
stream-based setting, we assume that in both settings the suggested elements,
along with their hidden responses, are drawn i.i.d.~from some unknown source
distribution. We then ask under what conditions, and at what cost, can a
stream-based algorithm obtain the same output distribution as a given black-box
pool algorithm. Such an exact emulation is advantageous, as it allows direct application of methods and results developed for the pool-based setting, in the stream-based setting. Especially, if a pool-based algorithm succeeds in practice, but its analysis is unknown or limited, exact emulation guarantees that success is transferred to the stream setting as well.

For discrete source distributions, any pool-based algorithm can be emulated in
a stream-based setting, simply by waiting long enough, until the desired element shows up again. The challenge for stream-based interactive algorithms is thus to achieve the same output distribution as a pool-based algorithm, while observing as few suggested elements as possible. Clearly, there are many cases in which it is desired to require less suggested elements: this could result in saving of resources such as time, money, and communication. In active learning as well, while examples are usually assumed cheap, they are not usually completely free in all respects.

We study emulation of pool-based algorithm in two settings. First, we consider the fully general case. We provide a stream algorithm that can emulate any given black-box pool algorithm, and uses a uniformly bounded expected number of observed elements. The bound on the expected number of observed elements is exponential in the number of selected elements. We further prove a lower bound which indicates that this exponential dependence is necessary. Second, we consider  \emph{utility-based} interactive algorithm for the pool setting. We provide a stream algorithm that emulates such pool algorithms, using repeated careful solutions of the well known ``Secretary Problem'' \citep{Dynkin63, GilbertMo66, Ferguson89}. The expected number of observed elements for this algorithm is only linear in the number of selected elements. In this case too we prove a matching lower bound.

Finally, we show a lower bound that applies to active learning for binary classification. We conclude that even in this well-studied setting, there are cases in which there exists a significant gap between the best pool-based algorithm and the best stream-based algorithm. This result  generalizes a previous observation of \cite{GonenSaSh13b} on the sub-optimality of CAL \citep{CohnAtLa94}, the classical stream-based active learning algorithm, compared to pool algorithms.

This paper is structured as follows: In \secref{def} formal definitions and notations are provided. \secref{simple} discusses natural but suboptimal solutions. \secref{gen} provides an algorithm and a lower bound for the general case, and \secref{utility} addresses the case of utility-based pool algorithms. In \secref{active} we provide a lower bound that holds for active learning for binary classification. We conclude in \secref{conclusion}. Some of the proofs are provided in \appref{proofs}.

\section{Definitions} \label{sec:def}

For a predicate $p$, denote by $\one[p]$ the indicator function which is $1$ if $p$ holds and zero otherwise. For an integer $k$, denote $[k] := \{1,\ldots,k\}$. For a sequence $S$, $S(i)$ is the $i$'th member of the sequence. Denote concatenation of sequences by $\circ$. 
For $A,B$ which are both sequences, or one is a set and one a sequence, we use $A =_\pi B$ and $A \subseteq_\pi B$ to denote equality or inclusion on the unordered sets of elements in $B$ and in $A$.

Let $\cX$ be a measurable domain of elements, and let $\cY$ be a measurable domain of responses. 
A pool-based (or just pool) interactive algorithm $\cA_p$ receives as input an integer $q \leq m$, and a pool of elements $(x_1,\ldots,x_m) \in \cX^m$. We assume that for each $x_i$ there is a response $y_i \in \cY$, which is initially hidden from $\cA_p$. Denote $S = ((x_i,y_i))_{i \in [m]}$. For a given $S$, $S_X$ denotes the pool $(x_1,\ldots,x_m)$. At each round, $\cA_p$ selects one of the elements $i_t$ that have not been selected yet, and receives its response $y_{i_t}$.
After $q$ rounds, $\cA_p$ terminates. Its output is the set $\{(x_{i_1},y_{i_1}),\ldots,(x_{i_q},y_{i_q})\}$. For a pool algorithm $\cA_p$, denote by $\outpool(S,t)$ the element that $\cA_p$ selects at round $t$, if $S$ is the pool it interacts with. $\outpool(S,t)$, which can be random, can depend on $S_X$ and on $y_{i_k}$ for $k < t$.
Denote by $\outpool(S,[t])$ the sequence of elements selected by $\cA_p$ in the first $t$ rounds. $\outpoolr(S,t)$ and $\outpoolr(S,[t])$ similarly denote the selected elements along with their responses. The final output of $\cA_p$ is the set of pairs in the sequence $\outpoolr(S,[q])$. We assume that $S \mapsto \outpoolr(S,[q])$ is measurable.

We assume that the pool algorithm is permutation invariant. That is, for any $S,S' \subseteq (\cX\times\cY)^m$, if $S'$ is a permutation of $S$ then $\outpool(S,[q]) = \outpool(S',[q])$, or if $\cA_p$ is randomized then the output distributions are the same. When the pool $S$ is drawn i.i.d.~this does not lose generality. 

A stream-based (or just stream) interactive algorithm $\cA_s$ receives as input an integer $q$. We assume an infinite stream $S \subseteq (\cX \times \cY)^\infty$, where $S(t) = (x_t,y_t)$. At iteration $t$, $\cA_s$ observes $x_t$, and may select one of the following actions:
\begin{itemize}
\item Do nothing
\item Select $x_t$ and observe $y_t$
\item Terminate.
\end{itemize}
At termination, the algorithm outputs a subset of size $q$ of the set of pairs $(x_t,y_t)$ it observed. Denote by $\outstream(S,t)$ the $t$'th element that $\cA_s$ selects and is also in the output set. Denote by $\outstream(S,[t])$ the sequence of first $t$ elements selects and are also in the output set. Use $\outstreamr$ to denote the elements along with their responses.
The output of $\cA_s$ when interacting with $S$ is the set of the pairs in the sequence $\outstreamr(S,[q])$. We assume $S \mapsto \outstreamr(S,[q])$ is measurable. The total number of elements selected by $\cA_s$ when interacting with $S$ (including discarded elements) is denoted $\nselect(\cA_s,S,q)$. The number of iterations (observed elements) until $\cA_s$ terminates is denoted $\niter(\cA_s,S,q)$.

We look for stream algorithms that emulate pool algorithms. We define an
equivalence between a stream algorithm and a pool algorithm as follows.
\begin{definition}
Let $\cD$ be a distribution over $\cX \times \cY$ and let $q$ be an integer. 
Let $S \sim \cD^m,S' \sim \cD^\infty$. A pool algorithm $\cA_p$ and a stream algorithm $\cA_s$ are \emph{$(q,\cD)$-equivalent}, if the total variation distance between the distributions of $\outpoolr(S,[q])$ and $\outstreamr(S',[q])$ is zero.
\end{definition}
Denote by $\cD_X$ the marginal of $\cD$ on $\cX$. 
Below, unless specified otherwise, we assume that the probability under $\cD_X$ of observing any single $x \in \cX$ is zero. This does not lose generality, since if this is not the case, $\cD_X$ can be replaced by the distribution $\cD_X \times \mathrm{Unif[0,1]}$, with the interactive algorithms ignoring the second element in the pair.

\newcommand{\algwait}{
\begin{algorithm}[h]
\caption{Algorithm $\cA_{\text{wait}}$}
\label{alg:wait}
\begin{algorithmic}[1]
\STATE In the first $m$ iterations, observe $x_1,\ldots,x_m$ and do nothing.
\STATE $S \leftarrow ((x_1,\star),\ldots,(x_m,\star))$
\STATE $j \leftarrow 1$
\REPEAT
\STATE In iteration $t$, observe element $x_t$
\IF{$x_t = \outpool(S,j)$}
\STATE Select $x_t$ and observe $y_t$
\STATE $S(i) \leftarrow (x_t, y_t)$
\STATE $j \leftarrow j+1$.
\ENDIF
\UNTIL{$j = q+1$}
\STATE Return the set of all the pairs $(x,y)$ in $S$ with $y \neq \star$.
\end{algorithmic}
\end{algorithm}
}

\newcommand{\algnowait}{
\begin{algorithm}[h]
\caption{Algorithm $\cA_{\text{nowait}}$}
\label{alg:nowait}
\begin{algorithmic}[1]
\REQUIRE Pool size $m$, Black-box pool algorithm $\cA_p$.
\STATE In each iteration $t \in [m]$, select $x_t$ and observe $y_t$.
\STATE Return the pairs in $\outpoolr(S,q)$. 
\end{algorithmic}
\end{algorithm}

}

\section{Simple equivalent stream algorithms} \label{sec:simple}

Let $\cA_p$ be a pool algorithm. For any discrete distribution $\cD$ over $\cX \times \cY$, and any $q$, it is easy to define a stream algorithm which is $(q,\cD)$-equivalent to $\cA_p$. Let ``$\star$'' be some value not in $\cY$, and define $\cA_{\text{wait}}$ as in \algref{wait}.

\algwait

This stream algorithm is $(q,\cD)$ equivalent to $\cA_p$ for any discrete distribution $\cD$, and it has $\nselect(\cA_{\text{wait}},S',q) = q$ for all $S' \in (\cX \times \cY)^\infty$. However, $\E_{S' \sim \cD^\infty}[\niter(\cA_{\text{wait}},S',q)]$ is not bounded for the class of discrete distributions. 

On the other hand, the stream algorithm $\cA_{\text{nowait}}$ defined in \algref{nowait} is also $(q,\cD)$ equivalent to $\cA_p$.
We have $\niter(\cA_{\text{wait}},S',q) = m$ for all $S' \in (\cX \times \cY)^\infty$, the same as the pool algorithm. However, also $\nselect(\cA_{\text{nowait}},S',q) = m > q$.
These two simple approaches demonstrate a possible tradeoff between the number of selected elements and the number of iterations when emulating a pool algorithm. 

\algnowait

\section{An equivalent algorithm with a uniform bound on expected iterations} \label{sec:gen}

We present the stream algorithm $\cA_{\textrm{gen}}$ (see \algref{gen}), which can emulate any pool based algorithm $\cA_p$ using only black-box access to $\cA_p$. 
The algorithm emulates a general pool algorithm, by making sure that in each iteration, its probability of selecting an element is identical to the conditional probability of the pool algorithm selecting the same element, conditioned on the history of elements and responses selected and observed so far. This is achieved by repeatedly drawing the remaining part of the pool, and keeping it only if it is consistent with the elements that were already selected. We further can use the partial pool draw only if the element to be selected happens to have been observed last.

\begin{algorithm}
\caption{Algorithm $\cA_{\textrm{gen}}$}
\label{alg:gen}
\begin{algorithmic}[1]
\REQUIRE Original pool size $m$, label budget $q < m$, black-box pool algorithm $\cA_p$.

\STATE $S_0 \leftarrow ()$
\FOR{$i = 1:q$}
\REPEAT \label{step:rep}
\STATE Draw $m-i+1$ elements, denote them $\bar{x}_{i,i},\ldots,\bar{x}_{i,m}$.
\STATE $S'_i \leftarrow ((\bar{x}_{i,i},\star),\ldots,(\bar{x}_{i,m},\star))$.
\UNTIL \label{step:unt} $\outpoolr(S_{i-1}\circ S_i',[i-1]) =_\pi S_{i-1}$ and $\outpool(S_{i-1}\circ S_i',i) = \bar{x}_{i,m}$. 
\STATE Select $\bar{x}_{i,m}$, get the response $\bar{y}_{i,m}$.
\STATE $S_i \leftarrow S_{i-1} \circ ((\bar{x}_{i,m},\bar{y}_{i,m}))$.
\ENDFOR
\STATE Output $S_q$.
\end{algorithmic}
\end{algorithm}

Below we show that $\cA_{\textrm{gen}}$ improves over the two stream algorithms presented above, in that it selects exactly $q$ elements, and has a uniform upper bound on the expected number of iterations, for any source distribution.
First, we prove that $\cA_{\textrm{gen}}$ indeed emulates any pool-based algorithm. The proof is provided in \appref{proofs}.

\begin{theorem}\label{thm:genequiv}
For any pool algorithm $\cA_p$, any distribution $\cD$ over $\cX \times \cY$, any integer $m$ and $q \leq m$, $\cA_s := \cA_{\textrm{gen}}(\cA_p)$ is $(q,\cD)$-equivalent to $\cA_p$. 
\end{theorem}

The next theorem provides an upper bound on the expected number of elements observed by $\cA_{\textrm{gen}}$. Unlike $\cA_{\textrm{wait}}$, this upper bound holds uniformly for all source distributions. 
\begin{theorem}\label{thm:geniter}
For any pool algorithm $\cA_p$, any distribution $\cD$ over $\cX \times \cY$, any integer $m$ and $q \leq m$, if $\cA_s := \cA_{\textrm{gen}}(\cA_p)$, 
$\nselect(\cA_s,S,q) = q$ for any $S \in (\cX \times \cY)^\infty$, and
\[
\E_{S \sim \cD^\infty}[\niter(\cA_s,S,q)] \leq m^2 \left(\frac{em}{q-1}\right)^{q-1}.
\]
\end{theorem}

\begin{proof}
First, clearly $\nselect(\cA_s,S,q) = q$ for any $S \sim \cD^\infty$.
We now prove the upper bound on the expected number of iterations of $\cA_s$. Let $S \sim \cD^m$. 
For $i \geq 1$, $z_1,\ldots,z_{i-1} \in \cX$, denote $Z_j = \{z_1,\ldots,z_j\}$, and let

\[
p_i(z_1,\ldots,z_i) := \P[\outpool(S,[i]) =_\pi Z_{i} \mid Z_i \subseteq_\pi   S_X].
\]

Suppose that $(S_{i-1})_X =_\pi Z_{i-1}$. The expected number of times that steps \ref{step:rep} to \ref{step:unt} are repeated for index $i$ is the inverse of the probability that the condition in \ref{step:unt} holds. This condition, in our notation, is that $\outpool(S_{i-1}\circ S_i',[i-1]) =_\pi Z_{i-1}$ and $\outpool(S_{i-1}\circ S_i',i) = \bar{x}_{i,m}$. 
We have, from the permutation invariance of $\cA_p$,
\begin{align*}
&\P[\outpool(S_{i-1}\circ S_i',[i-1]) =_\pi Z_{i-1} \mid (S_{i-1})_X =_\pi Z_{i-1}] = p_{i-1}(z_1,\ldots,z_{i-1}).
\end{align*}
In addition, for every draw of $S_i'$,
\[
\P[\outpool(S_{i-1}\circ S_i',i) = \bar{x}_{i,m} \mid \outpool(S_{i-1}\circ S_i',[i-1]) =_\pi Z_{i-1} \wedge (S_{i-1})_X =_\pi Z_{i-1}] = \frac{1}{m-i+1}.
\]
This is since under the conditional, one of the elements in $S_i'$ must be selected by $\cA_p$ in round $i$. Therefore, the probability that the condition in step \ref{step:unt} holds is $p_{i-1}(z_1,\ldots,z_{i-1})/(m-i+1)$. The expected number of times that steps \ref{step:rep} to \ref{step:unt} are repeated for index $i$ is the inverse of that, and in each round $m-i+1$ elements are observed. Therefore the expected number of elements observed until selection $i$ is made conditioned on $z_1,\ldots,z_{i-1}$ is $(m-i+1)^2/p_{i-1}(z_1,\ldots,z_{i-1})$. The unconditional expected number of elements observed until selection $i$ is $(m-i+1)^2\cdot \E[1/p_{i-1}(\outstream(S',[i-1]))]$.

For a set of indices $J$, denote $S|_J = \{ S(j) \mid j \in J\}$.
\begin{align*}
\E[1/p_i(\outstream(S',[i])] &= \E[1/p_i(\outpool(S,[i])]\\
&= \int_{\{z_1,\ldots,z_i\} \subseteq \cX \times \cY} d\P[\outpool(S,[i]) =_\pi Z_i]\cdot \frac{1}{p_i(z_1,\ldots,z_i)} \\
&= \int_{\{z_1,\ldots,z_i\} \subseteq \cX \times \cY}d\P[Z_i \subseteq_\pi S_X],
\end{align*}
Hence
\begin{align*}
&\E[1/p_i(\outstream(S',[i])]\leq  \int_{\{z_1,\ldots,z_i\} \subseteq \cX \times \cY}\sum_{J \subseteq [m], |J| = i} d\P[(S|_J)_X = Z_i]\\
&=\sum_{J \subseteq [m], |J| = i}\int_{\{z_1,\ldots,z_i\} \subseteq \cX \times \cY} d\P[(S|_J)_X = Z_i]\\
&= \sum_{J \subseteq [m], |J| = i} 1 = \binom{m}{i}.
\end{align*}

It follows that the expected number of elements observed after the $i-1$'th selection and until selection $i$ is at most $(m-i+1)^2\binom{m}{i-1}$.
We conclude that
\[
\E[\niter(\cA_s, S, q)] \leq \sum_{i=0}^{q-1} (m-i)^2\binom{m}{i} \leq m^2 \left(\frac{em}{q-1}\right)^{q-1}.
\]
This completes the proof.
\end{proof}

From the existence of $\cA_{\textrm{gen}}$ we can conclude that the pool-based and the stream-based setting are essentially equivalent, up to the number of observed elements. However, the expected number of observed elements is exponential in $q$. In the next section we show that this exponential dependence cannot be avoided for general pool algorithms.

\subsection{A lower bound for expected number of iterations}

We provide a lower bound, which shows that for some pool algorithm, any equivalent stream algorithm has an expected number of observed elements which is at least exponential in $q$. This indicates that not much improvement can be achieved over $\cA_{\textrm{gen}}$ for the class of all pool-based algorithms. The proof involves constructing a pool-based algorithm in which the last selected element determines the identity of the previously selected elements. This is easy in a pool setting, since the algorithm has advance knowledge of all the available elements. In a stream setting, however, this requires a possibly long wait to obtain the matching last element. Because the stream algorithm is allowed to select elements in a different order than the pool algorithm, additional care is taken to make sure that in this case, it is not possible circumvent the problem this way.
The proof of \thmref{lowergen} is provided in \appref{proofs}.

\begin{theorem}\label{thm:lowergen}
There is an integer $q_0$ and a constant $C > 0$, such that for $q \geq q_0$, if $4q^2 \log(4q) \leq m$, then there exist a pool algorithm $\cA_p$ and a marginal $\cD_X$, such that any stream algorithm $\cA_s$ which is $(q,\cD)$ equivalent to $\cA_p$ for all $\cD \in \distset(\cD_X)$, and selects only $q$ elements, has 
\[
\exists \cD \in \distset(\cD_X),  \E_{S \sim \cD^\infty}[\niter(\cA,S,q)] \geq C\left(\frac{ m}{q^2\log(4q)}\right)^{\frac{q-1}{2}}.
\]
\end{theorem}

\section{Utility-based pool algorithms} \label{sec:utility}
$\cA_{\textrm{gen}}$ gives a uniform guarantee on expected the number of iterations, however this guarantee is exponential $q$. We now consider a more restricted  class of pool algorithms, and show that it allows emulation with an expected number of iterations linear in $q$.

A common approach for designing pool-based interactive algorithms, employed, e.g., in \citet{SeungOpSo92,LewisGa94, TongKoller02, GuoGr07,GolovinKrRa10,GuilloryBi10,GolovinKr11, GonenSaSh13b, CuongLeYe14}, is to define a utility function, that scores each element depending on the history of selected elements and their responses so far. In each round, the algorithm selects the element that maximizes the current utility function. We consider black-box emulation for this class of pool-based algorithms.

Formally, a utility-based interactive pool algorithm is defined by
a utility function $\cU$, of the form \mbox{$\cU:\cup_{n=0}^{\infty}(\cX\times \cY)^n \times \cX \rightarrow \reals$.} $\cU(x,S_{t-1})$ is the score of element $x$ given history $S_{t-1}$. The pool algorithm selects, in each round, the element that is assigned the maximal score by the utility function given the history. We assume for simplicity that there are no ties in $\cU$. The utility-based interactive pool algorithm for $\cU$, denoted $\cA_p^\cU$, is defined in \algref{pool}.

\begin{algorithm}[h]
\caption{$\cA_p^\cU$}
\label{alg:pool}
\begin{algorithmic}[1]

\REQUIRE Elements $x_1,\ldots,x_m$, budget $q < m$.
\STATE $S_0 \leftarrow ()$
\STATE $M_0 \leftarrow [m]$
\FOR{$t = 1:q$}
\STATE $i_t \leftarrow \argmax_{j \in M_{t-1}} \cU(x_j,S_{t-1})$.
\STATE Select $x_{i_t}$, get  $y_{i_t}$.
\STATE $S_t \leftarrow S_{t-1} \circ (x_{i_t},y_{i_t})$. 
\STATE $M_t \leftarrow M_{t-1} \setminus \{i_t\}$.
\ENDFOR
\STATE Output the set of all pairs in $S_q$.
\end{algorithmic}
\end{algorithm}

\subsection{An stream algorithm for utility-based pool algorithms}
We propose a stream algorithm $\cA_s^\cU$ that emulates utility-based pool algorithms $\cA_p^\cU$. We stress that we do not attempt to maximize the value of $\cU$ on selected elements, but to emulate the behavior of the pool algorithm that uses $\cU$. This is because we do not assume any specific relationship between the value of the utility function and the reward of the algorithm. For instance, the utility-based pool algorithm might be empirically successful although its analysis is not fully understood \citep[e.g.][]{TongKoller02}.

The definition of $\cA_s^\cU$ uses the solution to the well-known 
\emph{secretary problem} \citep{Dynkin63,GilbertMo66,Ferguson89}. In the classical
formulation of this problem, an algorithm sequentially observes a stream of $n$
real numbers, and selects a single number. The goal of the algorithm is to select the maximal number out of the $n$, but it can only select a number immediately after it is observed, before observing more numbers. It is assumed that the $n$ numbers in the stream are unknown and selected by an adversary, but their order of appearance is uniformly random. The goal is to select the maximal number with a maximal probability, where $n$ is known to the algorithm.

\newcommand{\psec}{p_{\mathrm{sp}}}
This task can be optimally solved by a simple deterministic algorithm, achieving a success probability $\psec(n)$, which satisfies $\lim_{n \rightarrow \infty}\psec(n) = 1/e$. The optimal algorithm observes the first $t(n)$ numbers, then selects the next observed number which is at least as large as the first $t(n)$. The limit of $t(n)/n$ for $n \rightarrow \infty$ is $1/e$.

Given a stream of size $k$ of real values $R = (r_1,\ldots,r_k)$, we say that $\secp(n,R)$ holds if the optimal solution to the secretary problem for size $n$ selects $r_k$ after observing the stream prefix $R$.
$\cA_s^\cU$ is given in \algref{stream}. It uses repeated applications of the solution to the secretary problem to retrieve each of the selected elements. Because the solution succeeds with a probability less than $1$, its application might fail. This can be identified in retrospect. In this case, a new solution is selected. This trial-and-error approach means that $\cA_s^\cU$ usually selects more than $q$ elements. However the expected number of selected elements is a constant factor over $q$.

To make sure the equivalence holds, $\cA^s_\cU$ never selects an element that could not have been in a pool in which the previous elements have been selected. This is achieved by discarding such elements in each round. The upper bound on the expected number of observed elements bounds the expected number of elements discarded in this way.

\begin{algorithm}[h]
\caption{$\cA^{s}_{\cU}$}
\label{alg:stream}
\begin{algorithmic}[1]
\STATE $L_0 \leftarrow ()$
\STATE $\cX_1 = \cX$
\FOR{$i = 1:q$}
\REPEAT
\FOR{$j = 1:m-i+1$} \label{step:secprattempt}
\STATE Repeatedly draw elements from $\cD_X$, until drawing an element in $\cX_i$. \\Denote it $x_{i,j}$, and let $r_{i,j} \leftarrow \cU(x_{i,j},L_{i-1})$.
\IF{$\secp(m-i+1,(r_{i,1},\ldots,r_{i,j}))$}
\STATE $k\leftarrow j$
\STATE Select $x_{i,k}$, get its response $y_{i,k}$. 
\ENDIF
\ENDFOR
\UNTIL $r_{i,k} = \max\{r_{i,1},\ldots,r_{i,m-i+1}\}$.
\STATE $k_i \leftarrow k$
\STATE $L_i \leftarrow L_{i-1} \circ  (x_{i,k_i},y_{i,k_i})$. 
\STATE $\cX_{i+1} \leftarrow \{x \in \cX_i \mid \cU(x,L_{i-1}) < \cU(x_{i,k_i},L_{i-1})\}$
\ENDFOR
\STATE Output the set of pairs in $L_q$.
\end{algorithmic}
\end{algorithm}

First, we show that $\cA_s^\cU$ is indeed equivalent to $\cA_p^\cU$. The proof is provided in \appref{proofs}.
\begin{theorem}\label{thm:equivsec}
For any utility function $\cU$, any distribution $\cD$ over $\cX \times \cY$, any integer $m$ and $q \leq m$, $\cA_s^\cU$ is $(q,\cD)$-equivalent to $\cA_p^\cU$.
\end{theorem} 

The following theorem give an upper bound on the expected number of selected elements and the expected number of observed elements used by $\cA_s^\cU$. 
\begin{theorem}\label{thm:goodsec}
For any utility function $\cU$, any distribution $\cD$ over $\cX \times \cY$, any integer $m$ and $q \leq m$, 
\[
\E_{S \sim \cD^\infty}[\nselect(\cA_s^\cU,S,q)] = \psec^{-1}(m) q,
\]
and
\[
\E_{S \sim \cD^\infty}[\niter(\cA_s^\cU,S,q)] \leq \psec^{-1}(m)\exp(\frac{q}{{m-q}})\cdot qm.
\]
\end{theorem}
For $q \leq m/2$, and $m \rightarrow \infty$, it follows from \thmref{goodsec} that the expected number of selected elements is $eq$, and the expected number of observed elements is at most \mbox{$e^2 q m $}.

\begin{proof}[of \thmref{goodsec}]
Call a full run of the loop starting at step \ref{step:secprattempt} an attempt for the $i$'th element. In each attempt for the $i$'th element, $m-i+1$ elements from $\cX_i$ are observed. The expected number of attempts for each element $i$ is $e$, since each attempt is a run of the secretary problem, with a success probability of $\psec(m)$.
Therefore, the expected number of elements from $\cX_i$ observed until $x_i$ is selected is $\psec^{-1}(m)\cdot(m-i+1)$. 

Denote by $f_i$ the utility function $\cU(\cdot, L_{i-1})$. Let $x_i := x_{i,k_i}$, be the $i$'th element added to $L_i$. Then $\cX_{i} = \{ x \in \cX_{i-1} \mid f_{i-1}(x) \leq f_{i-1}(x_{i-1})\}$.

Consider the probability space defined by the input to the stream algorithm $S \sim \cD^\infty$, and let $Z_i,Z'_i \sim \cD_X$ for $i \in [q]$ such that these random variables and $S$ are all independent. Denote
 \[
p(\alpha,i) := \P[f_i(Z_i)\leq \alpha \mid Z_i \in \cX_i].
\]
$p(\alpha,i)$ is a random variable since $\cX_i$ depends on $S$. 
Let $U_i := p(f_i(Z'_i),i)$. Since we assume no ties in $\cU$, and no single $x$ has a positive probability in $\cD_X$, then conditioned on $\cX_i$, $U_i$ is distributed uniformly in $[0,1]$. Hence $U_1,\ldots,U_q$ are statistically independent. 

For $i > 1$, define the random variable $M_i := p(f_{i-1}(x_{i-1}),i-1)$. Then $M_i = \P[\cX_i]/\P[\cX_{i-1}]$. The expected number of elements that need to be drawn from $\cD$ to get a single element from $\cX_i$ is $1/\P[\cX_i] = (\prod_{j=1}^{i} M_j)^{-1}$. Therefore, 
\[
\E[\niter(\cA_s^\cU,S,q) \mid M_2,\ldots,M_q] = \sum_{i=1}^q \frac{\psec^{-1}(m)\cdot(m-i+1)}{\prod_{j=1}^{i} M_j}.
\]

The element $x_i$ maximizes the function $x \mapsto f_i(x)$ over $m-i+1$ independent draws of elements $x$ from $\cD_X$ conditioned on $x \in \cX_i$, hence it also maximizes $x \mapsto p(f_i(x),i)$. Therefore, for $i > 1$, $M_i$ is the maximum of $m-i+2$ independent copies of $U_i$, hence $P[M_i \leq p] = p^{m-i+2}$. 
Hence 
\[
dP[M_2,\ldots,M_q](p_2,\ldots,p_q)/dp_2\cdot\ldots \cdot dp_q = \prod_{i=2}^q dP[M_i \leq p_i]/dp_i =\prod_{i=2}^q (m-i+2)p_i^{m-i+1}.
\]

We have
\begin{align*}
\E[\niter(\cA_s^\cU,S,q)] &= \int_{M_2=0}^1 \ldots \int_{M_q=0}^1\E[\niter(\cA_s^\cU,S,q) \mid M_1,\ldots,M_q]dP[M_1,\ldots,M_q]\\
&=\int_{M_2=0}^1 \ldots \int_{M_q=0}^1 \sum_{i=1}^q \frac{\psec^{-1}(m)\cdot(m-i+1)}{\prod_{j=1}^{i} M_j}\prod_{l=2}^q (m-l+2)M_l^{m-l+1}dM_l\\
&=\sum_{i=1}^q\psec^{-1}(m)\cdot(m-i+1)\int_{M_2=0}^1 \ldots \int_{M_q=0}^1  \prod_{l=2}^i (m-l+2)M_l^{m-l}dM_l\\
&\hspace{18em}\cdot \prod_{l=i+1}^q (m-l+2)M_l^{m-l+1}dM_l,
\end{align*}
Therefore
\begin{align*}
\E[\niter(\cA_s^\cU,S,q)] &=\sum_{i=1}^q\psec^{-1}(m)\cdot(m-i+1)\prod_{l=2}^i \frac{m-l+2}{m-l+1}\\
&=\sum_{i=1}^q\psec^{-1}(m)\cdot(m-i+1)\prod_{l=2}^i(1 + \frac{1}{m-l+1})\\
&\leq \psec^{-1}(m)\cdot q m(1 + \frac{1}{m-q})^q \leq \psec^{-1}(m)\cdot\exp(\frac{q}{m-q})\cdot mq.
\end{align*}
This concludes the proof.
\end{proof}

\subsection{A lower bound for expected number of iterations}
The following lower bound shows that the expected number of observed elements required by \algref{stream} cannot be significantly improved by any emulation of general utility-based pool algorithms. This theorem holds for stream algorithms that select exactly $q$ elements, while \algref{stream} selects approximately $eq$ elements. We conjecture that even if allowing a constant factor more element selections, one can achieve at most a constant factor improvement in the expected number of observed elements.

The proof of the lower bound follows by constructing a utility function which in effect allows only one set of selected elements, and has an interaction pattern that forces the stream algorithm to select them in the same order as the pool algorithm. 
For a given distribution $\cD_X$ over $\cX$, let $\distset(\cD_X)$ be the set of distributions over $\cX \times \cY$ such that their marginal over $\cX$ is equal to $\cD_X$. The proof of \thmref{utilitylowerbound} is provided in \appref{proofs}.
\begin{theorem}\label{thm:utilitylowerbound}
For any $m \geq 8$, $q \leq m/2$,
there exists a utility-based pool algorithm, and a marginal $\cD_X$, such that any stream algorithm $\cA_s$ which is $(q,\cD)$ equivalent to the pool algorithm for all $\cD \in \distset(\cD_X)$, and selects only $q$ elements, has 
\[
\exists \cD \in \distset(\cD_X),  \E_{S \sim \cD^\infty}[\niter(\cA_s,S,q)] \geq \frac{q}{8}\floor{\frac{m}{2\log(2q)}}.
\]
\end{theorem}

\section{Active Learning for Binary Classification}\label{sec:active}
In active learning for binary classification, recent works provide relatively
tight label complexity bounds, that hold for both the stream-based and the
pool-based settings. In \cite{BalcanLo13}, tight upper and lower bounds for
active learning of homogeneous linear separators under log-concave distributions are provided. The bounds hold for
both the stream-based and the pool-based setting, and with the same bound on
the number of unlabeled examples. In \cite{HannekeYa15}, tight minimax label
complexity bounds for active learning are provided for several classes of
distributions. These bounds also hold for both the stream-based and the
pool-based setting. In that work no restriction is placed on the number of
unlabeled examples.  

These results leave open the possibility that
for some distributions, a pool-based algorithm with the same label complexity
as a stream-based algorithm might require significantly fewer unlabeled examples. In \exref{exampleh} and \thmref{active} we show that this is indeed the case.

\begin{example}\label{ex:exampleh}
For given integers $m$ and $q \leq m$, and $T \leq q$, define $\cX = \{a_{k,j} \mid k \in [q], j \in \{0,\ldots,2^{\min(k,T)-1}-1\}\} \cup \cX'$, where $\cX'$ includes arbitrary elements so that $|\cX| = n$, for some $n \geq q2^T/2$.
Define the following hypothesis class $\cH \subseteq \cY^\cX$.
\begin{equation}\label{eq:exampleh}
\cH := \{ h_{i} \mid i \in \{0,\ldots,2^q-1\}\}\text{, where }h_i(a_{k,j}) = 
\begin{cases} 
\one[i \bmod 2^k = j] & k \leq T,\\
\one[\,\lfloor i/2^{T-k} \rfloor \bmod 2^T = j].  & k > T.
\end{cases}
\end{equation}
Essentially, for $k \leq T$, $h_i(a_{k,j}) = 1$ if the $k$ least significant bits in the binary expansion of $i$ are equal to the binary expansion of $j$ to $T$ bits. For $k \geq T$, $h_i(a_{k,j}) = 1$ if $T$ consecutive bits in $i$, starting from bit $T-k$, are equal to the binary expansion of $j$.
\end{example}

\begin{theorem}\label{thm:active}
Let $q \geq 22$ and $m \geq 8\log(2q)q^2$ be integers. 
Consider \exref{exampleh} with $m,q$, setting $T = \ceil{\log_2(q)}$ and $n = \floor{m/7\log(2q)}$. Consider $\cH$ as defined in \eqref{exampleh}. There exist $\delta,\epsilon \in (0,1)$ such that there is a pool-based active learning algorithm that uses a pool of $m$ unlabeled examples and $q$ labels, such that for any distribution $\cD$ which is consistent with some $h^* \in \cH$ and has a uniform marginal over $\cX$, with a probability of at least $1-\delta$, $\P[\hat{h}(X) \neq h^*(X)] \leq \epsilon$.
On the other hand, for $q > 22$, any stream-based active learning algorithm with the same guarantee requires at least $\frac{q}{32}\floor{\frac{m}{7\log(2q)}}$ unlabeled examples in expectation. 
\end{theorem}

The proof is provided in \appref{proofs}.
This result shows that a gap between the stream-based and the pool-based settings exists not only for general interactive algorithms, but also specifically for active learning for binary classification. 

The gap is more significant when $q = \tilde{\Theta}(\sqrt{m})$, and can be as large as $\tilde{\Omega}(m^{3/2})$ unlabeled examples in a stream, versus $m$ that are required in a pool. It has been previously observed \citep{GonenSaSh13b} that in some cases, a specific pool-based active learning algorithm for halfspaces is superior to the classical stream-based algorithm CAL \citep{CohnAtLa94}. \thmref{active} shows that this is not a limitation specifically of CAL, but of any stream-based active learning algorithm. 

The upper bound in \thmref{goodsec} for utility-based pool algorithms can be applied for several deterministic pool-based active-learning algorithms which use a utility function \cite[e.g.,][]{GolovinKr11, GonenSaSh13b, CuongLeYe14}. 
The upper bound shows that when the label budget $q$ is relatively small, the gap between the stream and the pool settings is not significant. For instance, consider an active learning problem in which a utility-based pool active learner achieves a label complexity close to the information-theoretic lower bound for the realizable setting \citep{KulkarniMiTs93}, so that $q \in \Theta(\log(1/\epsilon))$. The passive learning sample complexity is at most $m \in \Theta(1/\epsilon)$. Therefore, a stream-based active learner with the same properties needs at most $O(\log(1/\epsilon)/\epsilon)$ unlabeled examples. Therefore, in this case the difference between the pool-based setting and the stream-based setting can be seen as negligible.

\section{Conclusions} \label{sec:conclusion}

In this work we studied the relationship between the stream-based and the pool-based interactive settings, by designing algorithms that emulate pool-based behavior in a stream-based setting, and proving upper and lower bounds on the stream sizes required for such emulation. Our results concern mostly the case where the label budget of the stream algorithm is similar or identical to that of the pool algorithm. We expect that as the label budget grows, there should be a smooth improvement in the expected stream length, which should approach $m$ as the label budget approaches $m$. There are many open problems left for further work. Among them, whether it is possible to emulate utility based pool algorithms with a linear stream size in $q$ and exactly $q$ labels, and a relaxation of the requirement for exact equivalence, which would perhaps allow using smaller streams.

\subsection*{Acknowledgements}
This work was supported in part by the Israel Science Foundation (grant No. 555/15).

\bibliographystyle{abbrvnat}
\bibliography{\mysharedbib}

\appendix
\section{Additional Proofs}\label{ap:proofs}

Several proofs use the following lemma.
\begin{lemma}\label{lem:berexpect}
Let $\alpha \in (0,\half), p \in (0,\alpha^2/2)$. Let $X_1,X_2,\ldots$ be independent Bernoulli random variables with $\P[X_i = 1] \leq p$. Let $I$ be a random integer, which can be dependent on the entire sequence $X_1,X_2,\ldots$. Suppose that $\P[X_I = 1] \geq \alpha$. Then $\E[I] \geq \frac{\alpha^2}{2p}$.
\end{lemma}

\begin{proof}
$\E[I]$ is minimized under the constraint when $\P[X_i = 1] = p$. Therefore assume this equality holds. Let $W$ be the random variable whose value is the smallest integer such that $X_W = 1$. 
Let $T$ be the largest integer such that $\P[W \leq T] \leq \alpha$.

The expectation of $I$ is lower bounded subject to $\P[X_I = 1] \geq \alpha$ by $I$ such that
$\P[I = W \mid W \leq T] = 1$, $\P[I = W \mid W = T+1] = \alpha - \P[W \leq T]$, and in all other cases, $I = 0$. Therefore,
\[
\E[I] \geq \E[W \cdot \one[W \leq T]].
\]
We have
\begin{align*}
  \frac{1}{p} &= \E[W] = \E[W \cdot \one[W \leq T]] + \E[W \cdot \one[W > T]]\\
 &= \E[W \cdot \one[W \leq T]] + (\frac{1}{p} + T)(1-p)^T.
\end{align*}
Therefore
\[
\E[I] \geq \E[W \cdot \one[W \leq T]] = \frac{1}{p} - (\frac{1}{p} + T)(1-p)^T.
\]
From the definition of $T$, $T$ is the largest integer such that $1-(1-p)^T \leq \alpha$. Hence $T \geq  \frac{\log(1-\alpha)}{\log(1-p)}$ and $(1-p)^T \leq (1-\alpha)/(1-p)$. Therefore
\[
\E[I] \geq \frac{1}{p} - \left(\frac{1}{p} + \frac{\log(1-\alpha)}{\log(1-p)}\right)\frac{1-\alpha}{1-p} \geq \frac{1}{p} - \left(\frac{1}{p} - \frac{\log(1-\alpha)}{2p}\right)\frac{1-\alpha}{1-p}
\]
Hence
\[
p \E[I] \geq 1 + \frac{1-\alpha}{1-p}(\log(1-\alpha)/2 - 1)
\]
For $p \leq a^2/2$ and $\alpha \in (0,1/2)$, elementary calculus shows that $p \E[I] \geq \alpha^2/2$.
\end{proof}

\begin{proof}[of \thmref{genequiv}]
Consider the probability space defined by the infinite sequence $S' \sim \cD^\infty$ which generates the input to the stream algorithm, and an independent sequence $S \sim \cD^m$ which is the input to the pool algorithm. 

For $z_1,\ldots,z_q \in \cX \times \cY$, denote $Z_j = \{z_1,\ldots,z_j\}$. We have, for every $i \in [q]$, 
\begin{align*}
&d\P[\outpoolr(S,[i]) =_\pi Z_i]=\\
&\quad \sum_{j=1}^i d\P[\outpoolr(S,i) = z_{j} \mid \outpoolr(S,[i-1]) =_\pi Z_i \setminus \{ z_j\}]\cdot d\P[\outpoolr(S,[i-1]) =_\pi Z_i \setminus \{ z_j\}].
\end{align*}
The same holds for $\outstreamr(S',\cdot)$. 
To show the equivalence it thus suffices to show that for all $z_1,\ldots,z_q \in \cX \times \cY$, $i \in [q]$, 
\[
d\P[\outstreamr(S',i) = z_i \mid \outstreamr(S',[i-1]) =_\pi Z_{i-1}] = 
d\P[\outpoolr(S,i) = z_i \mid \outpoolr(S,[i-1]) =_\pi Z_{i-1}].
\]

From the definition of $\cA_s$ we have 
\begin{align*}
&d\P[\outstreamr(S',i) = z_i \mid \outstreamr(S',[i-1]) =_\pi Z_{i-1}] \\
&=d\P[\outpoolr(S_{i-1}\circ S_i',i) = z_i \mid S_{i-1}=_\pi Z_{i-1} \wedge \outpoolr(S_{i-1}\circ S_i',[i-1]) =_\pi Z_{i-1}]\\
&=d\P[\outpoolr(S,i) = z_i \mid \outpoolr(S,[i-1]) =_\pi Z_{i-1}].
\end{align*}
The last equality follows since $\cA_p$ is permutation invariant and never selects the same index twice. This proves the equivalence.
\end{proof}

\begin{proof}[of \thmref{lowergen}]
Denote by $\Pi_k$ the set of permutations over $[k]$. 
Let the domain of elements be $\cX = [0,2]$ and assume responses in $\cY = \{0,1\}$. We now define a pool algorithm as follows.
Call a pool $S_X$ in which exactly one element in the pool is in $(1,2]$ and the rest are in $[0,1]$ a ``good pool''. On bad pools, $\cA_p$ always selects only elements in $[0,1]$ or only elements in $(1,2]$. 

For a good pool, denote for simplicity the single element in $(1,2]$ by $x_m$, and other elements by $x_1,\ldots,x_{m-1}$, where $x_{i-1} < x_i$ for $i \in [m-1]$.
Define a mapping $\psi: (1,2] \rightarrow \Pi_{m-1}$, such that if $x_m$ is uniform over $(1,2]$, then for $\psi(x_m)$ all permutations in the range are equally likely. 

$\cA_p$ behaves as follows: Let $\sigma = \psi(x_m)$. The first $q-1$ elements it selects are $x_{\sigma(1)},\ldots,x_{\sigma(q-1)}$. The last element it selects is $x_m$ if the response for all previous elements was $0$, and $x_{\sigma(q)}$ otherwise. 

Define the marginal $\cD_X$ over $\cX$ in which for $X \sim \cD_X$, $\P[X \in [0,1]] = 1-1/m$, $\P[X \in (1,2]] = 1/m$, and in each range $[0,1],(1,2]$, $X$ is uniform. The probability of a good pool under $\cD \in \distset(\cD_X)$ is $(1-1/m)^{m-1} \geq 1/e^2 =: p_g$. 
We now show a lower bound on the expected number of iterations of a stream algorithm which is $(q,\cD)$-equivalent to any $\cD \in \distset(\cD_X)$. 
Let $\cD_0$ be the distribution over $\cX \times \cY$ such that for $(X,Y) \sim \cD_0$, $X \sim \cD_X$ and $Y = 0$ with probability $1$. 
Let $S \sim \cD_0^m$ be the input to $\cA_p$.  

The proof will follow a series of claims:
\begin{enumerate}
\item The probability that, on a good pool, $\psi(x_m)$ is in a given set of permutation $\Phi(Z)$, where $Z$ is the set of first $q-1$ selected elements, is at least $1/2$.
\item When $\cA_s$ emulates a good pool, it selects an element from $(1,2]$ only after selecting $q-1$ elements from $[0,1]$. 
\item Therefore, when $\cA_s$ emulates a good pool, the expected number of observed elements until selecting the last element is lower bounded, and so the overall expected number is lower bounded.
\end{enumerate}

We start with \textbf{claim 1}. For a given set $Z = \{z_1,\ldots,z_{q-1}\} \subseteq [0,1]$, define the set of permutations $\Phi(Z)$ as follows. The expected number of elements that are smaller than $z_i$ in $S_X \sim \cD_X^m$, if $Z \subseteq_\pi  S_X$, is $n_i = (m-q)z_i + \sum_{j=1}^{q-1} \one[z_j < z_i]$. Let $\epsilon := \sqrt{(m-q)\log(4q)/2}$, and define
\begin{equation}\label{eq:phi}
\Phi(Z) := \{ \sigma \in \Pi_{m-1} \mid \exists \sigma' \in \Pi_{q-1}, \forall i \in [q-1], |\sigma^{-1}(i) - n_{\sigma'(i)}| \leq \epsilon \}.
\end{equation}
These are the permutations such that the first $q-1$ elements according to the permutation are mapped from elements with ranks in $[n_i - \epsilon,n_i+\epsilon]$. 
For $x \in S_X$, denote by $r_S(x)$ the rank of $x$ in $S_X$, when the elements in $S_X$ are ordered by value. Since $\psi(\outpool(S,q))$ determines the choice of $Z$ from $S_X$, we have
\begin{align*}
&\P[\psi(\outpool(S,q)) \in \Phi(Z) \mid \outpool(S,[q-1])=_\pi Z \wedge S \text{ is good}]\\
 &\quad\geq 
\P[\forall i \in [q-1], |r_S(z_{i}) - n_{i}| \leq \epsilon \:\mid\: \outpool(S,[q-1])=_\pi Z \wedge S \text{ is good}] \\
&\quad=
\P[\forall i \in [q-1], |r_S(z_{i}) - n_{i}| \leq \epsilon \mid Z \subseteq_\pi S\wedge S \text{ is good}].
\end{align*}
The last inequality follows since $\psi(\outpool(S,q))$ is uniform over all permutations.  By Hoeffding's inequality, for any $i \leq q-1$,
\[
\P[ |r_S(z_{i}) - n_{i}| > \epsilon \mid Z \subseteq_\pi S \wedge S \text{ is good}] \leq 2\exp(-2\epsilon^2/(m-q)).
\]
Therefore, using the definition of $\epsilon$ and applying the union bound, we get, for any $Z \subseteq [0,1]$ with $|Z| = q-1$, 
\begin{equation}\label{eq:claim1}
\P[\psi(\outpool(S,q)) \in \Phi(Z) \mid \outpool(S,[q-1])=_\pi Z \wedge S \text{ is good}] \geq \half.
\end{equation}
This completes the proof of \textbf{claim 1}.

We now turn to \textbf{claim 2}. Consider a stream algorithm which is $(q,\cD)$-equivalent to $\cA_p$ for any $\cD \in \distset(\cD_X)$. Consider runs of $\cA_s$ with input $S' \sim \cD_0^\infty$. Denote by $E_g$ the event that the output of $\cA_s$ is equal to a possible output of $\cA_p$ on a good pool with $S \sim \cD_0^m$. Then $\P[E_g] \geq p_g$. Claim 2 is that 
\begin{equation}\label{eq:claim2}
\P[\outstream(S',[q-1]) \subseteq_\pi [0,1] \mid E_g] = 1.
\end{equation}
In other words, when simulating a good pool, the elements in $[0,1]$ are all selected before the element in $(1,2]$. 

To show claim 2, note that by the definition of $\cA_p$, for any source distribution over $\cX \times \cY$, if $\cA_p$ outputs a set with elements both in $[0,1]$ and in $(1,2]$, then there is exactly one element in $(1,2]$ in the output, and all the responses in the output for elements in $[0,1]$ are $0$ with probability $1$.

Now, suppose that $\P[\outstream(S',[q-1]) \subseteq_\pi [0,1] \mid E_g] < 1$. Then $\P[\outstream(S',q) \in [0,1] \mid E_g] > 0$, since there can be only one element in $(1,2]$ in the output of a good pool. But, consider running $\cA_s$ with a source distribution $\cD' \in \distset(\cD_X)$ such that for $(X,Y) \sim \cD'$, $X \sim \cD_X$ and $\P_{\cD'}[Y = 0|X=x] = \half$ for all $x$. There is a positive probability that in the first $q-1$ selected elements all the responses are $0$, just as for $\cD_0$. Therefore, also for $S'' \sim {\cD'}^\infty$, \mbox{$\P[\outstream(S'',q) \in [0,1] \mid E_g] > 0$}. But then there is a positive probability that the response for the last element, which is in $[0,1]$, is $1$, contradicting the $(q,\cD')$-equivalence of the pool and $\cA_s$. This proves \textbf{claim 2}.

We now show \textbf{claim 3} which completes the proof. From claim 2 in \eqref{claim2}, we conclude that $\P[\outstream(S',q) \in (1,2] \mid E_g] = 1.$
Therefore, from claim 1 in \eqref{claim1}, for any $Z \subseteq_\pi [0,1]$ with $|Z| = q-1$, 
\[
\P[\psi(\outstream(S',q)) \in \Phi(\outstream(S',[q-1])) \mid  E_g] \geq 1/2.
\]
Therefore
\[
\P[\psi(\outstream(S',q)) \in \Phi(\outstream(S',[q-1]))] \geq \P[E_g]/2 \geq p_g/2.
\]

Now,  let $X_i \sim \cD_X$ be the $i$'th element observed after selecting the first $q-1$ elements, and let $B_i = \one[\psi(X_i) \in \Phi(Z)]$, where $Z$ is the set of $q-1$ selected elements. $B_i$ are independent Bernoulli random variables, each with a probability of success at most $p$, where from the definition of $\phi$ in \eqref{phi}, 
\[
p \leq \frac{|\Phi(Z)|}{|\Pi_{m-1}|} \leq \left(\frac{(q-1)(2\epsilon+1)}{m-1}\right)^{q-1} \leq  \left(\frac{2q^2\log(4q)}{ m}\right)^{\frac{q-1}{2}}.
\]
Let $I$ be the number of elements $\cA_s$ observes after selecting $Z$, until selecting element $q$. We have $\P[B_I=1] \geq p_g/2$.
By \lemref{berexpect}, for $p \leq p_g^2/8$, $p \E[I] \geq p_g^2/8$. From the assumption in the theorem statement, $2q^2 \log(4q)/m \leq \half$, hence for a large enough $q$, $p \leq 2^{-(q-1)} \leq p_g^2/8$, and so $\E[I] \geq \frac{p_g^2}{8}p^{-1}$. Hence there is a constant such that 
\[
\E[I] \geq C\left(\frac{ m}{q^2\log(4q)}\right)^{\frac{q-1}{2}}.
\]
Since $\E[\niter(\cA,S,q)] \geq \E[I]$, this completes \textbf{claim 3} and finalizes the proof.
\end{proof}

\newcommand{\good}{\cG}

\begin{proof}[of \thmref{equivsec}]
Consider the probability space defined by $S \sim \cD^m$ and $S' \sim \cD^\infty$, where $S,S'$ are independent.
We prove the equivalence by showing that for any $j \in [q]$ and $L_j = ((x_{i,k_i},y_{i,k_i}))_{i \in [j]}$ that could have been selected by the pool algorithm, 
\[
d\P[\outpoolr(S,j+1) \mid \outpoolr(S,[j]) = L_j] = d\P[\outstreamr(S',j+1) \mid \outstreamr(S',[j]) = L_j].
\]
For a given $L_j$, denote by $\cD_{j+1}$ the distribution generated by drawing $(X,Y) \sim \cD$ conditioned on $X \in \cX_{j+1}$, where $\cX_{j+1}$ depends on $L_j$. Denote by $\good$ all the finite sequences of pairs such that when the optimal secretary problem solution is applied to the sequence, it succeeds. That is, the optimal value under the score $(x,y) \rightarrow \cU(x, L_j)$ is indeed selected. From the definition of $\cA_\cU^s$, we have
\begin{align*}
&d\P[\outstreamr(S',j+1) \mid \outstreamr(S',[j]) = L_j]= d\P_{\bar S \sim \cD_{j+1}^{m-j}}[\argmax_{(x,y) \in \bar S}\cU(x, L_j) \mid \bar S \in \good].
\end{align*}
For a given sequence $\bar S =((\bar{x}_i,y_i))_{i \in [m-j]}$, let $\sigma(\bar S):[m-j] \rightarrow [m-j]$ be a permutation such that for all $i \leq m-j$, $\bar{x}_{\sigma(i)} \leq \bar{x}_{\sigma(i+1)}$.
The success of the optimal secretary problem algorithm depends only on the ordering of ranks in its input sequence, hence there is a set of permutations $\good'$ such that $\bar S \in \good$ if and only if $\sigma(\bar S) \in \good'$. 
Now, $\argmax_{(x,y) \in \bar S}\cU(x, L_j)$ depends only on the identity of pairs in $\bar S$, while $\sigma(\bar S)$ depends only on their order. Since the elements in $\bar S$ are i.i.d., these two properties are independent. Therefore
\[
d\P_{\bar S \sim \cD_{j+1}^{m-j}}[\argmax_{(x,y) \in \bar S}\cU(x, L_j) \mid \bar S \in \good] = d\P_{\bar S \sim \cD_{j+1}^{m-j}}[\argmax_{(x,y) \in \bar S}\cU(x, L_j)].
\]
Therefore
\begin{align*}
&d\P[\outstreamr(S',j+1) \mid \outstreamr(S',[j]) = L_j]\\
&=d\P_{\bar S \sim \cD_{j+1}^{m-j}}[\argmax_{(x,y) \in \bar S}\cU(x, L_j) ]\\
 &= d\P_{\hat{S} \sim \cD^{m-j}}[\argmax_{(x,y) \in \hat{S}}\cU(x, L_j)  \mid \hat{S} \subseteq (\cX_{j+1} \times \cY)^{m-j}]\\
&= d\P_{\hat{S} \sim \cD^{m-j}}[\argmax_{(x,y) \in \hat{S}}\cU(x, L_j)  \mid \forall (x,y) \in \hat{S}, i \in [j],\:\: \cU(x,L_{i-1}) < \cU(x_{i,k_i},L_{i-1})]\\
&=d\P_{\hat{S} \sim \cD^{m-j}}[\argmax_{(x,y) \in \hat{S}}\cU(x, L_j)  \mid \outpoolr(L_j \circ \hat{S},[j]) = L_j]\\
&=d\P_{S \sim \cD^m}[\argmax_{(x,y) \in S\setminus L_j}\cU(x, L_j)  \mid \outpoolr(S,[j]) = L_j]\\
&=d\P_{S \sim \cD^m}[\outpoolr(S)(j+1)  \mid \outpoolr(S,[j]) = L_j].
\end{align*}
Here $L_i$ is the prefix of length $i$ of $L_j$.
Since this equality holds for all $j \in [q-1]$, $d\P[\outstreamr(S',[q])] = d\P[\bar O_q(S,[q])]$.
\end{proof}

\begin{proof}[of \thmref{utilitylowerbound}]
Let $n = \floor{\frac{m}{2\log(2q)}}$, and let $\cD_X$ be a uniform distribution over $\cX = \{a_i \mid i \in [n]\}$. Assume $\cY = \{0,1\}$. A pool of size $m$ then includes all elements in $A:=\{a_i \mid i \in [2q-1]\}$ with a probability of at least $\alpha \geq 1-(2q-1)\exp(-m/n) \geq 1-\frac{1}{2q}$.

Consider a utility function $\cU$ such that given a history of the form $((a_1,0),\ldots,(a_t,0))$ for $t \in [q-1]$, assigns a maximal score in $\cX$ to $a_{t+1}$, 
and given a history of the form $((a_1,0),\ldots,(a_{t-1},0),(a_t,1))$, for $t\in[q-1]$, assigns a maximal score in $\cX$ to $a_{q+t-1}$. 
Then, in a pool that includes all elements $a_1,\ldots,a_{2q-1}$, the pool algorithm based on $\cU$ behaves as follows: In every round, if all selected elements so far received the response $0$, it selects at round $t$ the element $a_t$. Otherwise, it selects the element $a_{q+t-1}$. 

Let $\cD_0$ be a distribution in which the response is deterministically zero.
If the distribution is $\cD_0$, $\cA_s$ selects $Z_0 = \{a_1,\ldots,a_q\}$ with a probability at least $\alpha$. 
Denote $\cD_{t}$ for $t \in [q]$, in which the response is deterministically zero for $X \in \{a_1,\ldots,a_q\}\setminus \{a_t\}$ and $1$ for $a_t$. For this distribution, the algorithm must select the elements in \mbox{$Z_t = \{a_1,\ldots,a_t,a_{q+t},\ldots,a_{2q-1}\}$} with a probability at least $\alpha$.

We show a lower bound on the probability that $\cA_s$ selects $a_1,\ldots,a_q$ in order when the input sequence is $S \sim \cD_0^\infty$. Denote this probability $\beta$, and the event that this occurs $E$.

Consider the random process defined by the input sequence $S \sim \cD_0^\infty$ and the randomness of $\cA_s$. Let $T$ be a random variable, such that $T$ is the smallest round in which the algorithm selects some $a_{t'}$, for $t' > T$, or $T=0$ if no such round exists. Since $\P[T \in [q]] = 1-\beta$, there exists some $t^* \in [q]$ such that $\P[T = t^*] \geq (1-\beta)/q$. Now, consider the distribution $\cD_{t^*}$. Define a sequence of pairs $\gamma(S)$ such that $S$ and $\gamma(S)$ have the same elements in the same order, and the responses in $\gamma(S)$ are determined by $\cD_{t^*}$ instead of by $\cD_0$. Clearly, $\gamma(S)$ is distributed according to $\cD_{t^*}^\infty$. 
Consider a run of the algorithm on $S$ and a parallel run (with the same random bits) on $\gamma(S)$. The algorithm selects the same elements for both sequences until the $T$'th selection, inclusive. But the $T$'th selection is some element in $\{a_{T+1},\ldots,a_q\}$. If $T = t^*$, then $Z_{t^*}$ does not include the element selected in round $T$. Since $\cA_s$ selects exactly the set $Z_{t^*}$ with a probability of at least $\alpha$, we have $\P[T = t^*] \leq 1-\alpha$. Therefore $(1-\beta) = \P[T \in [q]] \leq q(1-\alpha)$, hence $\beta \geq \half$. 

Let $W_i$ be the number of elements that $\cA_s$ observes after selecting element $i-1$, until observing the next element. 
Let $X_i \sim \cD_X$ be the $i$'th element observed after selecting the first $i-1$ elements, and let $B_i = \one[X_i) =a_i]$. $B_i$ are independent Bernoulli random variables with $\P[B_i = 1] = 1/n$, and $\P[B_{W_i} = 1] \geq \P[E] = \beta
 \geq \half$. By \lemref{berexpect}, if $\frac{1}{n} \leq \frac18$, $\E[W_i] \geq \frac{n}{8}$.

It follows that the expected number of iterations over $q$ selections is at least $\frac{qn}{8} = \frac{q}{8}\floor{\frac{m}{2\log(2q)}}.$
\end{proof}

\begin{proof}[of \thmref{active}]
Let $\cD_X$ be uniform over $\cX$. 
Let $E$ be the event that $\cX \nsubseteq_\pi S_X$, and define $\delta := \P_{S \sim \cD_X^m}[E].$ Define $\epsilon = 1/n$, so that $\P[\hat{h}(X) \neq h^*(X)] < \epsilon$ if and only if $\hat{h} = h^*$. Let $i^*$ such that $h^* = h_{i^*}$. 

First, a pool-based algorithm can achieve the required accuracy as follows: Let $j_t := i^* \bmod 2^t$ for $t \leq T$, and $j_t :=  \lfloor i^*/2^{T-t} \rfloor \bmod 2^T$ for $t \geq T$. If $E$ holds, then $t$'th element selected by the pool algorithm is $a_{t,j}$, where $j$ is obtained as follows: If $t \leq T$, $j = j_{t-1}$. If $t > T$, $j = \floor{j_{t-1}/2}$. In round $1$, $j = 0$ and the selected element is $a_{1,0}$. Inductively, in this strategy the algorithm finds the $t$'th least significant bit in the binary expansion of $i^*$ in round $t$, thus it can use $j_{t-1}$ to set $j$ for round $t$. Under $E$, after $q$ labels $i^*$ is identified exactly. This happens with a probability of $1-\delta$ for any $\cD$ with the uniform marginal $\cD_X$.

Now, let $\cD_h$ be a distribution with a uniform marginal over $\cX$ with labels consistent with $h \in \cH$. Consider a stream-based algorithm $\cA_s$, denote its output by $\bar{h}$ and its input by $S \sim \cD_{h^*}^\infty$. 

\newcommand{\ent}{\mathbb{H}}
Let $I$ be a random variable drawn uniformly at random from $\{0,\ldots,2^q-1\}$. Let $H = h_I$ be a hypothesis chosen uniformly at random from $\cH$. Consider the probability space defined by $I, S \sim \cD_H^\infty$, and the run of $\cA_s$ on $S$.
Let $(Z_1,Y_1),\ldots,(Z_q,Y_q)$ be the examples that $\cA_s$ receives and the labels it gets, in order. Let $Y = (Y_1,\ldots,Y_q)$. Let $\alpha = \P[Z_1 = a_{1,0} \mid S_X]$. If $Z_1 = a_{1,0}$, then $\P[Y_1 = 0 \mid S_X] = \half$. If $Z_1 \neq a_{1,0}$, then $\P[Y_1 = 0 \mid S_X] \geq 3/4$. Let $\ent$ be the base-2 entropy, and $\ent_b$ be the binary entropy.  Then $\ent_b(Y_1 \mid S_X) = \ent_b((\alpha+1)/4)$, and so
\begin{align*}
\ent(H \mid Y,S_X) &= \ent(H,Y \mid S_X) - \ent(Y_1 \mid S_X) - \ent(Y_1,\ldots,Y_q \mid Y_1,S_X)\\
&\geq q -\ent_b((\alpha+1)/4) - (q-1)\\
&= 1 - \ent_b((\alpha+1)/4). 
\end{align*}
From the Taylor expansion of the binary entropy around $1/2$, $\ent_b(p) \leq 1 - (1-2p)^2/2$, therefore $\ent(H \mid Y, S_X) \geq (1-\alpha)^2/8$.
We have $\P[\bar{h} \neq H] \leq \delta$, hence $\P_{S_X}[\P[\bar{h} \neq H \mid S_X] \leq 2\delta] \geq \half$. By Fano's inequality, for any $S_X$ such that $\P[\bar{h} \neq H \mid S_X] \leq 2\delta$,
\[
(1-\alpha)^2/8 \leq \ent(H \mid Y ,S_X) \leq \ent_b(2\delta) + 2\delta q \leq 2\delta (\log_2(\frac{1}{2\delta}) + 2+q).
\]
Where the last inequality follows from $\ent_b(p) \leq p\log_2(1/p) + 2p$.
From the definition of $\delta$, we have $\delta \leq |\cX|\exp(-m/n)$. Setting $T = \ceil{\log_2(q)}$, and noting that $|\cX| \leq q 2^T/2 \leq q^2$ and so $m \geq n \log(128 q^3 |\cX|)$, we have $\delta \leq \frac{1}{128q^3}$. 

Therefore, for $q \geq  22$, $1-\alpha \leq \frac{1}{2q}$.

It follows that $\P_{S_X}[\P[Z_1 \neq a_{1,0} \mid S_X] \leq \frac{1}{2q}] \geq 1/2$. Now, the same argument holds for any round $t$ conditioned on $I \mod 2^t = 0$ and $Z_1=a_{1,0},\ldots,Z_t=a_{t,0}$, since in this case after $t$ labels, the algorithm has $q-t$ queries left, and needs to select from $\cH'$, which is equivalent to $\cH$, with $q-t$ instead of $q$. Moreover, $\P[\bar{h} = H \mid I \mod 2^t = 0] \leq 1-\delta$ as well, since this holds for every $H$ individually. We conclude that for every $t \leq q$, with a probability at least $\half$ over $S_X$, 
\[
\P[Z_t \neq a_{t,0}\mid S_X, H = h_0] \leq \frac{1}{2q}.
\]
It follows that with a probability at least $\half$ over $S_X$, $\P[Z_1 = a_{1,0},\ldots,Z_q = a_{q,0} \mid S_X, H = h_0] \geq 1/2$. Hence $\P[Z_1 = a_{1,0},\ldots,Z_q = a_{q,0} \mid H = h_0] \geq 1/4$. 

Now, suppose $H = h_0$. Let $W_t$ be the number of elements that $\cA_s$ observes after selecting element $t-1$, until observing the next element. 
Let $X_j \sim \cD_X$ be the $j$'th element observed after selecting the first $t-1$ elements, and let $B_j = \one[X_j =a_{t,0}]$. $B_j$ are independent Bernoulli random variables with $\P[B_j = 1] = 1/n$, and $\P[B_{W_t} = 1] \geq \P[E] = \beta
 \geq \frac14$. By \lemref{berexpect}, if $\frac{1}{n} \leq \frac18$, then $\E[W_t] \geq \frac{n}{32}$.
It follows that the expected number of iterations over $q$ selections is at least $\frac{qn}{32} \geq \frac{q}{32}\floor{\frac{m}{7\log(2q)}}.$
\end{proof}

\end{document}